\documentclass[letterpaper, 10 pt, conference]{ieeeconf}  

\IEEEoverridecommandlockouts                              

\usepackage{enumitem}
\usepackage{graphicx}
\usepackage{amsmath}
\usepackage{framed}
\usepackage{amsthm}
\usepackage{comment}
\usepackage{ bbold }
\usepackage{xcolor}
\usepackage{url}

\newtheorem{theorem}{Theorem}
\newtheorem{lemma}{Lemma}

\newtheorem{proposition}{Proposition}

\newcommand{\tr}{T}

\newcommand{\indc}{\delta}

\newcommand{\Ad}{m}

\newcommand{\Cr}{C_r}
\newcommand{\bE}{\mathbb{E}}

\newcommand{\bR}{\mathbb{R}}

\newcommand{\xS}{\mu}

\newcommand{\sign}{\textnormal{sign}}

\newcommand{\Sr}{S^r}
\newcommand{\xr}{x^r}

\renewcommand{\Pr}{\mathbb{P}}

\title{\LARGE \bf
Online Learning with Adversaries: A Differential-Inclusion Analysis}

\author{Swetha Ganesh, Alexandre Reiffers-Masson, Gugan Thoppe
\thanks{S. Ganesh and G. Thoppe are with the Computer Science and Automation Department, Indian Institute of Science (IISc), Bengaluru 560012, India. SG's research was supported by the Prime Minister's Research Fellowship (PMRF). GT's research was supported in part by DST-SERB’s Core Research Grant CRG/2021/00833, in part by DST-SERB's SIRE Fellowship SIR/2022/000444, in part by IISc Startup grants SG/MHRD-19-0054 and SR/MHRD-19-0040, and in part by the Pratiksha Trust Young Investigator     award. GT also wishes to thank Michel Benaim for several insightful discussions related to this work during his visit to the University of Neuchatel, Switzerland.        {\tt\small Emails: swethaganesh@iisc.ac.in,  gthoppe@iisc.ac.in}}
\thanks{A. Reiffers-Masson is with the Computer Science Department, IMT Atlantique, 655 Av. du Technopôle, 29280 Plouzan\'e, France. {\tt\small alexandre.reiffers-masson@imt-atlantique.fr}}%
}

\begin{document}

\maketitle
\thispagestyle{empty}
\pagestyle{empty}

\begin{abstract}
We introduce an observation-matrix-based framework for fully asynchronous online Federated Learning (FL) with adversaries. In this work, we demonstrate its effectiveness in estimating the mean of a random vector. Our main result is that the proposed algorithm almost surely converges to the desired mean $\mu.$ This makes ours the first asynchronous FL method to have an a.s. convergence guarantee in the presence of adversaries. We derive this convergence using a novel differential-inclusion-based two-timescale analysis. Two other highlights of our proof include (a) the use of a novel Lyapunov function to show that $\mu$ is the unique global attractor for our algorithm's limiting dynamics, and (b) the use of martingale and stopping-time theory to show that our algorithm's iterates are almost surely bounded. 
\end{abstract}

\section{Introduction}
Federated Learning (FL) \cite{mcmahan2017communication} is a paradigm for multiple edge/client nodes to collaborate and iteratively solve some global problem with the help of a central server. It has therefore garnered significant interest in machine \cite{zhang2021survey} and reinforcement learning \cite{qi2021federated}. However, most existing FL methods do not account for failures or adversarial clients, making them ineffective in practice. Also, among those that do, a majority are synchronous \cite{bernstein2018signsgd, data2021byzantine, gorbunov2022variance, jin2020stochastic, wu2020federated, yin2018byzantine}: the server waits for inputs from a large number of clients before updating the global estimate.
These approaches again are impractical because many edge devices are frequently offline and, when that is not the case, the (inevitable) slow devices decide the overall performance. These issues have put the focus on asynchronous FL, wherein the server updates as soon as one node provides its input. Our work introduces a radically new family of such FL methods.

To our knowledge, there exist only four asynchronous FL methods in the literature: (i.) Kardam \cite{damaskinos2018asynchronous}, (ii.) Zeno++ \cite{xie2020zeno++},  (iii.) AFLGuard \cite{fang2022aflguard}, and (iv.) BASGD \cite{yang2021basgd}. The first three use a sophisticated scoring rule for filtering out malicious estimates. However, Kardam’s issue is that it drops many correct estimates during attacks. On the other hand, Zeno++ and AFLGuard require a separate validation dataset at the parameter server, which is undesirable from the privacy viewpoint. Finally, BASGD is asynchronous only at the client side: the time between successive server updates is dictated by the stragglers, like in synchronous FL methods. Our proposed approach has none of the above issues. 

We only consider the mean estimation FL problem here and discuss our proposed approach only  in that context. This problem involves $p \geq 1$ clients (a small but unknown subset of which are malicious), whose joint goal is to estimate the mean $\mu$ of a random vector $X \in \bR^d,$ where $d \leq p.$ At its core, our approach considers  a tall observation matrix $A \in \bR^{p \times d}$ and provides each node $i$ access to the IID samples of the random variable $Y(i) := a_i^\tr X,$ where $\tr$ denotes transpose and $a_i^\tr$ is the $i$-th row of $A.$ We task node $i$ to locally estimate the mean of $Y(i).$ Separately, at every instance $n \geq 0,$ the server is  tasked to pick a client at random and request for its current local estimate. A honest client is expected to provide its actual estimate; the malicious agent can act arbitrarily (it can even collude with other attackers). The server then is to immediately update its $\mu$-estimate using the gradient of $|a_i^\tr x - y_n(i)|,$ i.e., using one update step of the SGD algorithm that solves $\min_x \|Ax - \bE[Y]\|_1,$ where $\|\cdot\|_1$ is the $\ell_1$ norm. Clearly, the above algorithm is asynchronous since the server updates the $\mu$-estimate immediately upon receiving an input from a client.

The basis for our above approach is as follows. Suppose a matrix $A$ and the vector $b = A x_*$ are known, but not the vector $x_*.$ Then a natural way to recover $x_*$ is to solve the linear system $Ax = b.$ In \cite{fawzi2011secure}, a variant of this problem is discussed. There,  the goal is again to recover $x_*$ but assuming knowledge of only the vector $b'= b + e$ instead of $b.$ The vector $e$ is presumed sparse and represents a one-time malicious attack. Due to $e$’s sparsity, solving for $x$ that minimizes the $\ell_1$-norm of $Ax - b'$ is now the natural way to recover $x_*.$ A key result in \cite{fawzi2011secure} is that the observation matrix $A$ being robust (see \eqref{e:FTD.Cond} in our work), i.e., has suitable redundancy depending on $e$’s sparsity, is the necessary and sufficient condition for $x_*$ to be the unique solution to this $\ell_1$-minimization problem. The vector corresponding to $b$ in our setup is $\bE[Y].$ While $b'$ provides the value of $b$ in the non-attacked (but unknown) coordinates, $\bE[Y]$ is fully unknown in our case. Thus, our proposed approach above can be seen as a modification of the one in \cite{fawzi2011secure} that obtains online estimates of both $\bE[Y]$ and $\mu$ simultaneously. Note that the malicious agents in \cite{fawzi2011secure} attack only once. In contrast, in our case, since the server (unknowingly) will query every malicious node infinitely often during the algorithm's run, each such node  will have infinitely many opportunities to poison the $\mu$-estimation update rule.
 
Our main contributions can be summarized as follows.

\begin{enumerate}
    \item \textbf{Algorithm}: We propose a novel fully asynchronous FL algorithm and  demonstrate its effectiveness for mean estimation. Unlike the sophisticated filtering scheme or the non-corrupt private dataset of Kardam, Zeno++, and AFLGuard, our approach uses a observation matrix to handle adversaries. This matrix choice is not unique; thus, we have a family of algorithms for solving the same problem. Separately, since the gradient of the $\|\cdot\|_1$-error involves a sign function, our algorithm ensures that the impact of an adversarial node in each iteration is limited to a sign change.

    \item \textbf{Result}: Our main result is that $A$ being robust (as in \cite{fawzi2011secure}) is again a necessary and sufficient condition for our algorithm's iterates to converge to $\mu$ almost surely. This makes our proposed approach the first asynchronous FL algorithm to have an a.s. convergence guarantee in the presence of adversaries. 

    \item \textbf{Analysis}: Our analysis is novel compared to the existing FL literature. It builds on the Differential Inclusion (DI) and the two-timescale stochastic-approximation theory. The two-timescale part arises because our approach estimates $\bE[Y]$ and $\mu$ using two stepsize sequences that decay to $0$ at different rates. In contrast, we use a DI---a set-valued generalization of an Ordinary Differential Equation (ODE)---to mainly account for the multitude of choices available to an adversarial client in each iteration. DI theory is being used for the first time for analyzing an  algorithm in adversarial settings. There are two additional highlights of our proof.
    %


    %
    \begin{enumerate}
        \item \textit{Lyapunov Function}: Our algorithm is based on the gradient-descent idea for minimizing $\|Ax - \bE[Y]\|_1.$ Typically, for analyzing such a method, the natural Lyapunov function would have been the objective function. However, in our adversarial setting, we have been unable to verify this claim. We instead prove that $\|x - \mu\|_2^2$ behaves as a Lyapunov function.

        \item \textit{Boundedness of Iterates}: A key step in any ODE/DI based analysis \cite{borkar2009stochastic} of stochastic algorithms is to show that the algorithm's iterates are stable. In this work, we use a novel martingale and stopping time based approach to show that the algorithm's iterates are almost surely bounded. 
    \end{enumerate}

    
    %

    %
        
    
    
\end{enumerate}

\section{Setup, Algorithm, and Main Result}

We describe here the statistical problem we study, our proposed algorithm to solve it, and our main result that describes the limiting behavior of this algorithm.


\textbf{Setup}: $X \in \bR^d$ is a random variable with finite mean and finite covariance matrix entries. There are $p$ agents to collect statistics about $X,$ but an unknown subset $M,$ with $|M| \leq \Ad,$ are malicious or adversarial. Specifically, the $i$-th agent has access to samples of the random variable $Y(i) := a_i^\tr X,$ where $a_i \in \bR^d$ is a known deterministic vector. At time $n \geq 1,$ a central server picks index $i_n$  uniformly at random from $\{1, \ldots, p\}$ and queries agent $i_n$ for an independent sample of $Y(i_n).$ Agent $i_n$ returns an actual sample if it is non-adversarial, and an arbitrary real number otherwise (the value can change on each  query and can depend on the history\footnote{Such adversaries are commonly referred to as omniscient. 
}).  In either case, $Y_n(i_n)$ denotes the obtained sample. 

\textbf{Goal}: Develop an online algorithm to estimate $\mu := \bE[X]$ using the sequence $(Y_n(i_n)).$ 

\textbf{Algorithm}: Our approach is based on the gradient descent idea for minimizing $\|Ax - \bE[Y]\|_1.$ Starting from an arbitrary $x_0 \in \bR^d$ and $y_0 \in \bR^p,$ our proposed algorithm to learn $\mu$ at the central server is, for $n \geq 0,$
\begin{equation}
\label{e:OL.Adversarial}
    \begin{aligned}
        x_{n + 1} = {} & x_n + \alpha_n a_{i_{n + 1}} [\sign(y_n(i_{n + 1}) - a_{i_{n + 1}}^\tr x_n)] \\
        y_{n + 1} = {} & y_n + \beta_n[Y_{n + 1}(i_{n + 1}) - y_n(i_{n + 1})] u_{i_{n + 1}},
    \end{aligned}
\end{equation}
where $u_{i}$ is $i$-th column of the $p \times p$-identity matrix and, for any $r \in \bR,$
\begin{equation}
    \sign(r) =
    \begin{cases}
        - 1 & \text{ if $r < 0,$} \\
          0 & \text{ if $r = 0,$} \\
          1 & \text{ if $r > 0.$}
    \end{cases}
\end{equation}
%
In \eqref{e:OL.Adversarial}, the variables indexed by $n$ are known at time $n,$ while the ones by $n + 1$ are not. Note that the coordinates of $y_n$ corresponding to malicious nodes are directly fed into $x_n$'s update rule. 

\textbf{Assumptions}: Apart from the conditions on $X, (i_n),$ and $Y_n(i_n)$ stated in the setup, we presume that the matrix $A$ and stepsize sequences $(\alpha_n)$ and $(\beta_n)$ satisfy the following.
\begin{enumerate}[leftmargin=*]
    \item \textbf{Observation matrix}: The matrix $A$ is tall ($p > d$), has full column rank, and satisfies 
    \begin{equation}
        \label{e:FTD.Cond}
        \sum_{i \in K^c} |a_i ^\tr x| > \sum_{i \in K} |a_i^\tr x|
    \end{equation}
    for all $x \in \bR^d \setminus \{0\}$ and $K \subseteq \{1, \ldots, p\}$ with $|K| = m.$ 
    
    \item \textbf{Stepsize}: $(\alpha_n)$ and $(\beta_n)$ are monotonically decreasing positive reals such that $\max \{\alpha_0, \beta_0\} \leq 1,$ $\sum_{n \geq 0} \alpha_n = \sum_{n \geq 0}\beta_n = \infty,$ $\lim_{n \to \infty} \alpha_n/\beta_n = \lim_{n \to \infty}\beta_n = 0,$ and $\max\{\sum_{n \geq 0} \alpha_n^2, \sum_{n \geq 0} \beta_n^2, \sum_{n \geq 0} \alpha_n \gamma_n\}  < \infty,$ where $\gamma_n = \sqrt{\beta_n \ln (\sum_{k = 0}^n \beta_k)}.$ An example is $\alpha_n = n^{-\alpha},$ $\alpha \in (2/3, 1],$ and $\beta_n = n^{-\beta},$ $\beta \in (1/2, 1] \cap (2(1 - \alpha), \alpha).$

\end{enumerate}

Our main result is stated below and is derived using a DI-based set-valued analysis. As we discuss in Section~\ref{s:Intuition}, such an analysis is natural for \eqref{e:OL.Adversarial} due to its sub-gradient nature  and, importantly, the presence of adversaries. Let  $h: \bR^d \to 2^{\bR^d}$ (the power set of $\bR^d$) be given by 
\begin{equation}
    \label{e:x.Driving.Function}
    h(x) = \left\{\frac{1}{p}\sum_{i = 1}^p  a_i \lambda_i: (\lambda_1, \ldots, \lambda_p) \in \Lambda(x)\right\},
\end{equation}
where $\Lambda(x)$ includes all  $(\lambda_1, \ldots, \lambda_p)$ for which 
\[
    \lambda_i \in 
    \begin{cases}
    \{\sign(\bE [Y(i)]  - a_i^\tr x)\}, & i \in M^c\!  \text{ and }\! a_i^\tr x\! \neq \!\bE[Y(i)], \\
    [-1, +1],  & \text{otherwise.}
    \end{cases}
\]
\begin{theorem}
\label{thm:SRI.Main.Result}
The following statements hold.
    \begin{enumerate}
        \item \label{st:GASE} $\mu$ is the unique Globally Asymptotically Stable Equilibrium (GASE) for the DI
        \begin{equation}
        \label{e:lim.DI}
            \dot{x}(t) \in h(x(t)).
        \end{equation}
    
        \item \label{st:y_n.rate} There exists some constant $\Lambda > 0$ such that
        \[
            \limsup_{n \to \infty}     \frac{\|y_n - \bE[Y]\|_{M^c}}{\gamma_n} \leq \Lambda \qquad \text{a.s.,}
        \]
        where $\|y\|_{M^c} = \sqrt{\sum_{i \in M^c}y^2(i)}.$

            
    
        \item \label{st:xn.Conv} $x_n \to \mu$ a.s.
    \end{enumerate}
\end{theorem}






The DI in \eqref{e:lim.DI} corresponds to the update rule of $x_n$ in \eqref{e:OL.Adversarial} with $y_n(i) \equiv \bE[Y(i)]$ for $i \in M^c,$ and the sign expression replaced with an arbitrary value in $[-1, +1]$, otherwise. Our first result states that every solution of this DI will converge to $\mu,$ irrespective of the sign choices made at the adversarial nodes (in a continuous time sense). The second statement provides the asymptotic rate at which $|y_n(i) - \bE[Y(i)]| \to 0,$ $i \in M^c,$ on every sample path. While this result assumes that the stepsizes are square-summable, it can be extended to cover the case of even non-square summable stepsizes; see \cite{thoppe2021law} for details. Our third and final result states that the actual $(x_n)$ iterates in \eqref{e:OL.Adversarial} also behave like the solutions of \eqref{e:lim.DI} and almost surely converge to $\mu.$ However, because the \sign\ function is not continuous, this is not a simple consequence of the first two statements. Instead, we have to rely on a more complex two-timescale DI analysis, and a  separate boundedness result for $(x_n)$ based on the theory of martingales and stopping times.






\subsection{Motivation for a DI-based Analysis}
\label{s:Intuition}

%
%
    


\begin{figure}
    \centering
    \includegraphics[scale=0.5]{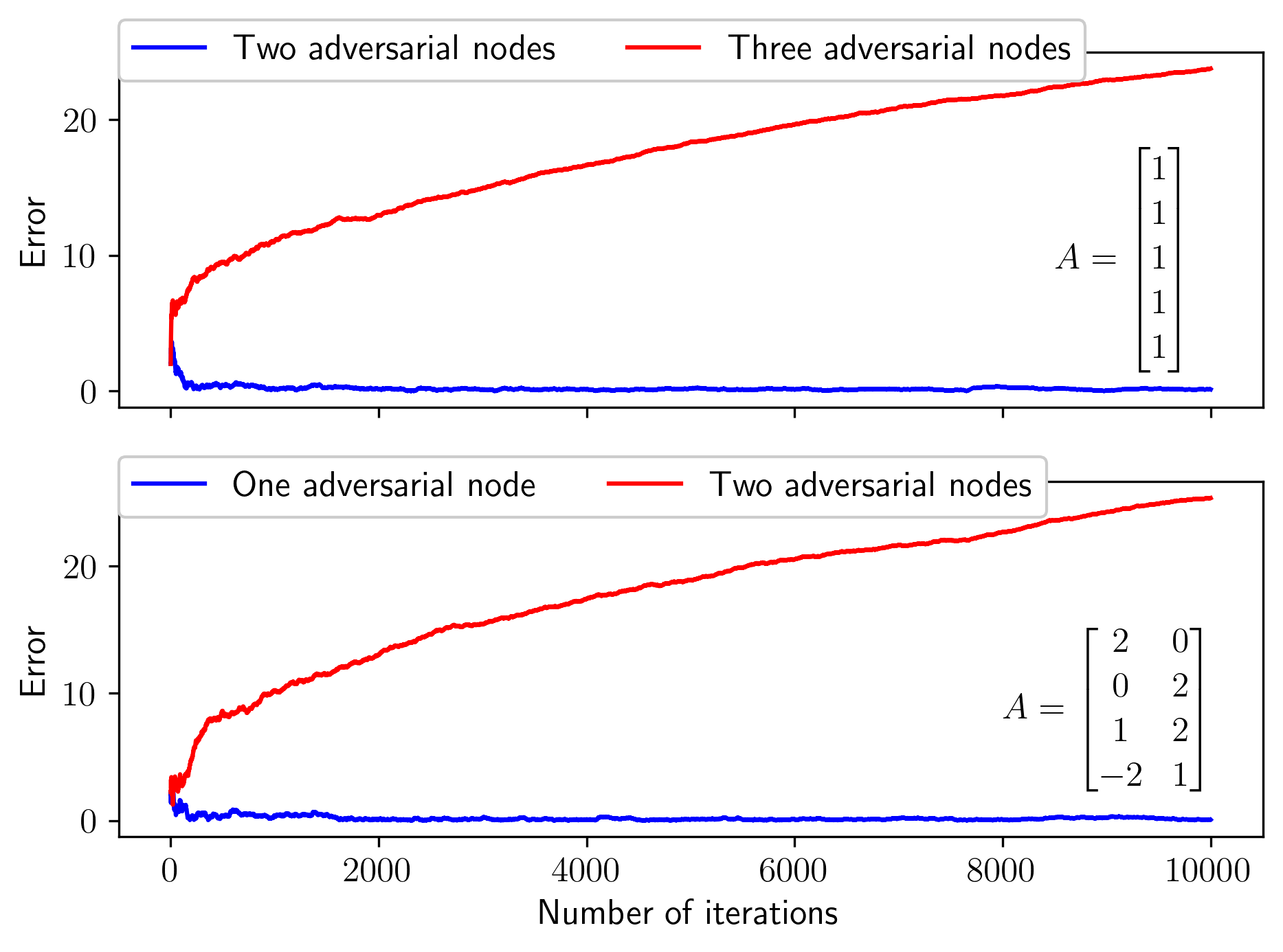}
    \caption{Error incurred by Algorithm~\ref{e:OL.Adversarial} ($\|x_n-\mu\|$) against the number of iterations ($n$). Within each subplot, the same measurement model is used with the only difference being the number of adversaries. The first subplot concerns the geometric median problem with $p=5$, while the second considers a generic matrix $A$ (see Section~\ref{s:Intuition}). 
    }
    \label{fig:ftd-eg}
\end{figure}



In this subsection, we give a simple example on why our algorithm will converge to $\mu$ even in the presence of adversarial measurements. We use a simplified set-up to illustrate 
the necessity of the DI analysis. 

Let $A$ be a vector of all ones. This implies that $\bE Y (i) = \mu\in\mathbb{R}$, for all $i$. Our problem setup then reduces to computing $x \in \bR$ that minimises $\sum_{i=1}^p |x-\bE Y (i)|$. The solution to this minimisation problem is called the geometric median \cite{wu2020federated}. 
Consider Algorithm \eqref{e:OL.Adversarial} in the deterministic setting, where all agents $i$ are given $\bE Y(i)$, instead of having to estimate it. Then, $y_n(i)$ will be $\mu$, for $i\in M^c,$ and any arbitrary value otherwise. It can be seen that the synchronous version of update \eqref{e:OL.Adversarial} can be written as
\begin{equation*}\label{eq: geom_median}
x_{n+1}=x_n+\alpha_n\Bigg[\underbrace{\sum_{i \in M^c} \lambda_n(i)}_{\text{Unperturbed subgradient }}+\sum_{i\in M}\underbrace{\lambda_n(i)}_{\text{Adversarial noise}}\Bigg],
\end{equation*}
where $\lambda_n(i) = \sign(\mu - x_n),$ if $i \in M^c$ and $x_n \neq \mu,$ and an arbitrary value in $[-1, 1],$ otherwise. Clearly, 
$\lambda_n(i)$ is the subgradient of $-|x-\bE Y (i)|$ when $i\in M^c$ and the perturbed subgradient given by the adversary, otherwise. The above update rule cannot be analysed using traditional ODE based approaches.  Firstly, the update can now take a set of values at each $x_n$. This is because $\lambda_n(i),$ $i \in M,$ can take any value in $[-1,1]$, regardless of $x_n$. Moreover, $\sign(\mu-x)$ is discontinuous at $x=\mu$, while ODE approaches require that this function be Lipschitz continuous.
Thus, the differential inclusion approach is preferred since it is capable of handling discontinuities and capturing the evolution of a set-valued map. The associated DI for the above update is given by:
\begin{equation*}
\dot x(t) \in \left\{|M^c|\sign(\mu-x)+\sum_{i\in M}v_i: v_i\in[-1,1] \right\},
\end{equation*}
when $x\neq \mu$ and 
\begin{equation*}
\dot x(t) \in \left\{\sum_{i= 1}^p v_i: v_i\in[-1,1] \right\},
\end{equation*}
when $x = \mu$.  The DI is modified at $x = \mu$ to make it continuous in a set-valued sense.

Note that if $|M^c|>|M|$ (equivalent to \eqref{e:FTD.Cond}), it follows that $\lim_{t\rightarrow +\infty} x(t) = \mu$. The intuition is as follows: if $\mu\neq x $, the sign of  $|M^c|\sign(\mu-x)+\sum_{i\in M}v_i$ will be always the same as the $\sign(\mu-x)$ and therefore the drift of the DI is controlled by the $\sign(\mu-x)$ and not by the adversaries. The performance of our algorithm for this problem with $p=5$ is shown in Figure~\ref{fig:ftd-eg}. Here, condition~\eqref{e:FTD.Cond} holds if $|M|=2$, but not when $|M|=3$. Consequently, our algorithm converges in the presence of two adversaries but diverges in the presence of three adversaries. 
 
More generally, condition~\eqref{e:FTD.Cond} is necessary and sufficient for our algorithm to converge. We emphasize that this condition is necessary even in the absence of noise and thus cannot be relaxed. A less obvious case where condition~\eqref{e:FTD.Cond} holds is given in the bottom subplot of Figure~\ref{fig:ftd-eg}. For the matrix $A$ in this example, the condition holds for $|M|=1$, but not when $|M|=2$.



\section{Proof of Theorem~\ref{thm:SRI.Main.Result}}
\label{sec:Proof}
We first discuss our proof strategy and then provide the details. Since  $y_n(i)$'s estimate for   $i \in M^c$ is not influenced by the $Y$ samples of other nodes, one would intuitively expect $\|y_n - \bE[Y]\|_{M^c} \to 0.$ Hence, \eqref{e:lim.DI} is the natural object for studying $(x_n)$'s behaviors. However, because the \sign\ function is discontinuous, $(x_n)$'s evolution cannot be  viewed as a simple perturbation of \eqref{e:lim.DI}'s solutions as in \cite[pg.~17]{borkar2009stochastic}. Instead, we rely on a two-timescale DI analysis  \cite{yaji2020stochastic}. Henceforth, $\|\cdot\|$ will denote the Euclidean norm. 

\subsection{Informal Outline of Two-timescale Analysis}
\label{s:Proof.Outline}
Our algorithm \eqref{e:OL.Adversarial} is of a two-timescale nature because $\alpha_n/\beta_n \to 0.$ Thus, the changes in the $x_n$ values eventually appear negligible compared to that of $y_n,$ which, in turn, implies $(x_n)$ and $(y_n)$'s behaviors can be studied in a decoupled fashion. Loosely, our analysis proceeds via the following prescribed steps from \cite{yaji2020stochastic}. 
\begin{enumerate}[leftmargin=*]
    \item $(y_n)$'s analysis: We set $x_n \equiv x$ for some arbitrary $x,$ and look at $y_n(i)$'s  evolution for $i \in M^c;$ we ignore what happens at the adversarial nodes. In our case, $y_n(i)$'s evolution is not influenced by the value of $x$ in any way. Further, its limiting ODE can be guessed to be $\dot{z}(t) = \frac{1}{p}(\bE[Y(i)] - z(t)).$ Since this scalar ODE is linear and has $\bE[Y(i)]$ as its unique GASE, it follows from a standard single-timescale stochastic approximation analysis \cite[Chapters~2 and 3]{borkar2009stochastic} that $|y_n(i) - \bE[Y(i)]| \to 0.$ 
  
    \item $(x_n)$'s analysis: From $(x_n)$'s perspective, $(y_n)$ would appear to have converged to its limit point. Accordingly, in $x_n$'s update rule, we now set $y_{n}(i) = \bE[Y(i)],$ for $i \in M^c,$ and allow for arbitrary values for adversarial $i$'s. This leads to the set-valued DI dynamics \eqref{e:lim.DI}. In the rest of this section, we formally prove that $\mu$ is its only attractor (Section~\ref{s:GASE}), the original $(x_n)$ sequence in \eqref{e:OL.Adversarial} is almost surely bounded (Section~\ref{s:Bounded.a.s.}), and it almost surely converges to $\mu$ (Section~\ref{s:rem.details}).
\end{enumerate}

\subsection{Analysis of the DI in \eqref{e:lim.DI}}
\label{s:GASE}
We first check that \eqref{e:lim.DI} is a well-defined DI. Recall that, for an (autonomous) ODE to be well-defined, one sufficient condition is that its driving function be Lipschitz continuous. In particular, this guarantees the existence and uniqueness of a solution for any initial point. Similarly, a DI is well-defined when its set-valued driving function $h$ is Marchaud, i.e., Lipschitz continuous in a set-valued sense (defined below). In general, solutions of a DI from a given starting point are not unique, but the above condition ensures existence. 

For $x \in \bR^d,$ let $Z(x):= M \cup \{i: a_i^T (x -\mu) =0 \}.$

\begin{lemma}
\label{lem:Marchaud}
The function $h$ defined in \eqref{e:x.Driving.Function} is Marchaud, i.e.,
    \begin{enumerate}
        \item $h(x)$ is convex and compact for all $x \in \bR^d;$
    
        \item $\exists K_h > 0$ such that, for all $x \in \bR^d,$ $
        \sup_{y \in h(x)} \|y\| \leq K_h(1 + \|x\|);$ and
    
        \item $h$ is upper semicontinuous or, equivalently,  $\{(x, \theta) \in \bR^d \times \bR^d: \theta \in h(x)\}$ is closed.
    \end{enumerate}
Hence, the DI in \eqref{e:lim.DI} is well-defined. 
\end{lemma}
\begin{proof}
The first two conditions are easy. For $h$'s upper semi-continuity, it suffices to check if $(x_n)$ and $(\theta_n)$ are such that $x_n \to x,$ $\theta_n \in h(x_n)$ $\forall n,$ and $\theta_n \to \theta,$ then $\theta \in h(x).$

For $i \in Z(x)^c$, $a_i^T (x -\mu)$ is either $>0$ or $<0.$ This fact along with $x_n \to x$ then implies $\exists n_0 \geq 0$ such that, for  $n \geq n_0,$ we have $\sign(a_i^T (x_n -\mu))=\sign(a_i^T (x -\mu))$ for all $i \in Z(x)^c$ and, hence, $Z(x)^c \subseteq Z(x_n)^c.$
%
%
%
Thus, $h(x_n) \subseteq h(x)$ for all $n \geq n_0,$ which implies  $(\theta_n)_{n \geq n_0} \subseteq h(x).$ The desired result now follows since $h(x)$ is compact. \hfill \qedsymbol
%
\end{proof}

We now show that $\mu$ is \eqref{e:lim.DI}'s unique GASE.

\hspace{2ex} \textit{Proof of Statement~\ref{st:GASE}, Theorem~\ref{thm:SRI.Main.Result}:}
It suffices to show that $V(x) = \frac{1}{2}\|x - \mu\|^2$ is a Lyapunov function \cite{benaim2005stochastic} for the DI in \eqref{e:lim.DI} with respect to $\{\mu\}.$  Clearly, $V(x) = 0$ if and only if $x = \mu.$ Further, for any $x \neq \mu$ and $\theta \equiv \frac{1}{p} \sum_{i = 1}^p a_i \lambda_i \in h(x),$
\begin{align}
    \nabla V&(x)^\tr \theta \nonumber\\
    = {} &  \frac{1}{p} \sum_{i = 1}^p \lambda_i a_i^\tr (x - \mu) \nonumber 
    \\
    = {} & \frac{1}{p} \left[ -\sum_{i \in M^c} |a_i^\tr (x - \mu)| + \sum_{i \in M} \lambda_i a_i^\tr (x - \mu)  \right] \label{e:Non.Adv.lambda.i} \\
    \leq {} & \frac{1}{p} \left[ -\sum_{i \in M^c} |a_i^\tr (x - \mu)| + \sum_{i \in M} |a_i^\tr (x - \mu)|  \right] \label{e:Adv.lambda.i.Bd} \\
    < {} & 0, \label{e:Ly.negative}
\end{align}
where 
\eqref{e:Non.Adv.lambda.i} holds since $\lambda_i a_i^\tr (x - \mu) = -|a_i^\tr(x - \mu)|$ for $i \in M^c,$  \eqref{e:Adv.lambda.i.Bd} is true because $r \leq |r|$ for any $r \in \bR$ and $|\lambda_i| \leq 1,$ while \eqref{e:Ly.negative} follows from \eqref{e:FTD.Cond} since $|M| \leq m$.

The claim now follows from \cite[Proposition~3.25]{benaim2005stochastic}. \hfill \qedsymbol

\subsection{Almost Sure Boundedness of $(x_n)$}
\label{s:Bounded.a.s.}
We use  martingale and stopping time theory to show that $(x_n)$ obtained using \eqref{e:OL.Adversarial} is almost surely bounded. 

Our proof needs a few intermediate results.  In relation to $(x_n)$ and $(y_n)$ in \eqref{e:OL.Adversarial}, define the following. For $n \geq 0,$ let
\begin{multline}
    b_n = \frac{1}{p} \sum_{i \in M^c} a_i [\sign(y_n(i) - a_i^\tr x_n) \\ -  \sign(\bE[Y](i) - a_i^\tr x_n)],
\end{multline} 
\begin{multline*}
g(x_n, y_n) = \frac{1}{p} \sum_{i \in M^c} a_i \sign(\bE[Y](i) - a_i^\tr x_n) \\ + \frac{1}{p} \sum_{i \in M} a_i \sign(y_n(i) - a_i^\tr x_n),
\end{multline*}
and
\begin{multline}
\label{e:Noise.Defn}
    M_{n + 1} = a_{i_{n + 1}}^\tr [\sign(y_n(i_{n + 1}) - a_{i_{n + 1}}^\tr x_n] \\ - g(x_n, y_n) - b_n.
\end{multline}
In the above terms, the update rule in \eqref{e:OL.Adversarial} can be written as 
\begin{equation}
\label{e:OL.Adversarial.Analytic.Form}
    x_{n + 1} = x_n + \alpha_n[g(x_n, y_n) + b_n + M_{n + 1}].
\end{equation}
Note that $g(x_n, y_n) \in h(x_n).$ Therefore, one can view $g(x_n, y_n)$ as the update direction that is prescribed by \eqref{e:lim.DI}, $b_n$ as a perturbation that arises since, for $i \in M^c,$ $y_n(i) \neq \bE[Y(i)]$ a.s. for any finite $n,$ and $M_{n + 1}$ as the noise.

\begin{lemma}
\label{lem:Int.Results}
The following statements are true.
\begin{enumerate}
    \item For $x \in \bR^d,$ let $\phi(x) = \frac{1}{p}\sum\limits_{i \in M^c} |a_i^\tr x| - \frac{1}{p}\sum\limits_{i \in M} |a_i^\tr x|.$ Then there exists  $\eta > 0$ such that $\phi(x) \geq \eta \|x\|$ \, $\forall x.$

    \item $|(x_n - \mu)^\tr b_n| \leq \frac{2\sqrt{|M^c|}}{p} \|y_n - \bE[Y]\|_{M^c}.$
    
    \item $(x - \mu)^\tr \theta \leq - \eta \|x - \mu\|$ for any $\theta \in h(x).$

    \item \label{st:x_n+1.x_n.Rel} Let $C_M := \sup_{1 \leq i \leq p} \|a_i\|.$ Then, for any $n \geq 0,$
    \begin{multline*}
        \|x_{n + 1} - \xS\|^2 \leq \|x_0 - \xS\|^2  + \sum_{k = 0}^n \alpha_k (x_k - \xS)^\tr M_{k + 1}\\ + \frac{2}{p} \sum_{k = 0}^n \alpha_k \|y_k - \bE[Y]\|_{M^c} + C_M^2 \sum_{k = 0}^n \alpha_k^2.
    \end{multline*}

\end{enumerate}
\end{lemma}
\begin{proof}
The first statement is trivially true for $x = 0.$ Hence, suppose $x \neq 0.$ It suffices to show that $\exists \, \eta > 0$ such that $\phi(x) \geq \eta$ for any $x$ with unit norm. However, this holds since (a) $\phi$ is continuous and $\{x \in \bR^d: \|x\| = 1\}$ is a compact set: thus, $\phi$ attains its minimum; and (b)  $\phi(x) > 0$ for any $x \neq \mu$ on account of  \eqref{e:FTD.Cond}.

For the second statement, note that
\[
    |\sign(r_1 - r_0) - \sign(r_2 - r_0)| \leq 2 \delta_{|r_1 - r_2| \geq |r_0 - r_2|}
\]
for any $r_0, r_1, r_2 \in \bR,$ where $\delta$ denotes the indicator function. Combining this with the fact that $\bE[Y(i)] = a_i^\tr \mu,$ for $i \in M^c,$ gives
\begin{align*}
    |(x_n & - \mu)^\tr b_n|  \\
    \leq {} & \frac{2}{p} \sum_{i \in M^c}  |a_i^\tr (x_n - \mu)| \indc_{|y_n(i) - \bE[Y(i)]| \geq |a_i^\tr x_n - a_i^\tr \mu|} \\
    \leq {} & \frac{2}{p} \sum_{i \in M^c} |y_n(i) - \bE[Y(i)]| \indc_{|y_n(i) - \bE[Y(i)]| \geq |a_i^\tr x_n - a_i^\tr \mu|} \\
    %
    %
    \leq {} & \frac{2}{p} \sum_{i \in M^c} |y_n(i) - \bE Y(i)| \\
    \leq {} & \frac{2 \sqrt{|M^c|}}{p} \|y_n - \bE[Y]\|_{M^c},
\end{align*}
as desired. 

We now discuss the third statement. Let $\theta \in h(x)$ be arbitrary. Then, 
\begin{align*}
    (x - & \mu)^\tr \theta \\
    \leq {} & \frac{1}{p} \left[- \sum_{i \in M^c} |a_i^\tr (x - \mu)| + \sum_{i \in M} |a_i^\tr (x - \mu)| \right]\\
    \leq {} & - \phi(x - \mu),
\end{align*}
where the first relation follows as in \eqref{e:Adv.lambda.i.Bd}, and the second relation holds from $\phi$'s definition. The claim now follows from our first statement above.



Finally, we derive the fourth statement. From \eqref{e:OL.Adversarial.Analytic.Form}, 
\begin{multline*}
    \|x_{n + 1} - \mu\|^2 = \|x_n - \mu\|^2 + \alpha_n^2 \|g(x_n, y_n) + b_n + M_{n + 1}\|^2 \\
    + 2 \alpha_n (x_n - \mu)^\tr [g(x_n, y_n)  + b_n + M_{n + 1}].
\end{multline*}
Statement 3 along with the fact that $g(x_n, y_n) \in h(x_n)$ shows $(x_n - \mu)^\tr g(x_n, y_n) \leq - \eta \|x_n - \mu\|,$ while Statement 2 gives the bound on $(x_n - \mu)^\tr b_n.$ Separately, $\|g(x_n, y_n) + b_n + M_{n + 1}\| = \|a_{i_{n + 1}}\| \leq C_M.$ It now follows that
\begin{multline*}
        \|x_{n + 1} - \xS\|^2 \leq \|x_n - \xS\|^2 - \alpha_n \eta \|x_n - \xS\| \\ + \frac{2 \sqrt{|M^c|} \alpha_n}{p} \|y_n - \bE[Y]\|_{M^c} + \alpha_n (x_n - \xS)^\tr M_{n + 1} + C_M^2 \alpha_n^2.
    \end{multline*}

The desired claim is now easy to see. \hfill 
\qedsymbol
\end{proof}

Presuming Statement~\ref{st:y_n.rate} in Theorem~\ref{thm:SRI.Main.Result} holds, we are now ready to show that $(x_n)$ is bounded almost surely, 

\begin{proposition}
\label{prop:stability}
$\sup\limits_{n \geq 0} \| x_n \| < \infty$ a.s.
\end{proposition}
\begin{proof}
Let $(\gamma_n)$ be as in Theorem~\ref{thm:SRI.Main.Result}. Fix an arbitrary integer $r \geq 1,$ and let $\Cr:= \frac{2r \sqrt{|M^c|}}{p}  \sum_{k = 0}^\infty \alpha_k \gamma_k + C_M^2 \sum_{k = 0}^\infty \alpha^2_k < \infty,$ and $T(r)$ be the stopping time $\inf\left\{n \geq 0: \frac{1}{\gamma_n} \|y_n - \bE[Y]\|_{M^c} > r\right\}.$  Next, for $n \geq 0,$ let
\[
    S_n = \|x_0 - \mu\|^2 + 2 \sum_{k = 0}^{n - 1} \alpha_k (x_k - \mu)^\tr M_{k + 1} + C_r.
\]
Clearly, $(S_n)$ and, hence, $(\Sr_n) \equiv (S_{n \wedge T(r)})$ is a martingale. 

Let $(\xr_n) \equiv (x_{n \wedge T(r)}).$ Then Statement 4 of Lemma~\ref{lem:Int.Results} shows $\|\xr_n - \mu\|^2 \leq \Sr_n$ $\forall \, n \geq 0.$ This implies $(\Sr_n)$ is a non-negative martingale and, hence, converges almost surely. Therefore, $(\xr_n)$ is bounded almost surely. 

Finally, note that 
\begin{align}
    E := {} & \bigg\{\sup_{n \geq 0} \|x_n\| = \infty\bigg\} \nonumber \\
    {} & \hspace{4em} \cap \left[ \bigcup_{r = 1}^{\infty} \bigg\{ \sup_{n \geq 0} \frac{\|y_n - \bE[Y]\|_{M^c}}{\gamma_n} \leq r \bigg\}\right] \nonumber \\
    = {} & \bigcup_{r = 1}^{\infty} \left\{\sup_{n \geq 0}\|\xr_n\| = \infty, \sup_{n \geq 0} \frac{\|y_n - \bE[Y]\|_{M^c}}{\gamma_n} \leq r \right\} \label{e:x_n.xr_n.s.swap} \\
    \subseteq {} & \bigcup_{r = 1}^{\infty} \bigg\{\sup_{n \geq 0}\|\xr_n\| = \infty \bigg\},\nonumber
\end{align}
where \eqref{e:x_n.xr_n.s.swap} follows from the fact that $\sup_{n \geq 0} \frac{\|y_n - \bE[Y]\|_{M^c}}{\gamma_n} \leq r$ implies $x_n = \xr_n$ for all $n.$ Since $(\xr_n)$ is almost surely bounded for any $r \geq 1,$ we get  $\Pr(E) = 0.$ From Statement~\ref{st:y_n.rate} in Theorem~\ref{thm:SRI.Main.Result}, we also have that 
\[
    \Pr\left(\bigcup_{r = 1}^\infty\left\{ \sup_{n \geq 0}  \frac{\|y_n - \bE [Y]\|_{M^c}}{\gamma_n}  \leq r\right\}\right) = 1.
\]
The desired claim now follows since, for any events $E_1$ and $E_2,$ $\Pr(E_1) = 1$ and $\Pr(E_2^c \cap E_1) = 0$ imply $\Pr(E_2) = 1.$
\hfill \qedsymbol
\end{proof}

\subsection{Rest of the Proof}
\label{s:rem.details}
In this section, we discuss the proofs of Statements~\ref{st:y_n.rate} and \ref{st:xn.Conv} of Theorem~\ref{thm:SRI.Main.Result}. 

Statement~\ref{st:y_n.rate} follows from \cite[Theorem~1]{pelletier1998almost}, which provides a law of iterated logarithm type result for generic stochastic approximation algorithms. That work assumes that the iterates almost surely converge, but this can be shown using the results in \cite[Chapters 2 and 3]{borkar2009stochastic}, as discussed in Section~\ref{s:Proof.Outline}.

To prove Statement~\ref{st:xn.Conv}, we rely on \cite[Theorem~4]{yaji2020stochastic}, which looks at convergence of generic two-timescale algorithms with set-valued limiting dynamics. Specifically, this latter result assumes $(x_n)$'s limiting DI has a global attractor (see A10 there), and states that, if ten other conditions (labelled A1 - A9 and A11 there) hold, then $x_n$ converges to this global attractor a.s. These ten conditions concern the behaviors of $x_n$ and $y_n$'s driving functions, stepsizes, and noise. Below we provide a brief commentary on why these assumptions hold for \eqref{e:OL.Adversarial}. The reader should note that the role of $x_n$ and $y_n$ is flipped in \cite{yaji2020stochastic}: the changes in $y_n$ eventually appear negligible compared to that of $x_n.$ The analysis there also accounts for Markov noise, but it can ignored using the approach suggested in Remark 3 there. Finally, for all of $(y_n)$'s analysis below, we ignore the evolution at adversarial nodes: instead, we account for them directly in the definition of the DI in \eqref{e:lim.DI}. 

Assumptions A1 and A2 of \cite{yaji2020stochastic} hold when the limiting DIs associated with $x_n$ and $y_n$ are Marchaud. For \eqref{e:OL.Adversarial}, this can be established like in the proof of our Lemma~\ref{lem:Marchaud}. Assumptions A3 and A4 concern Markov noise and, hence, trivially hold true in our case. Assumption A5 is on stepsizes and it holds in our case because we also assume those conditions. Assumption A8 there holds if the $(x_n)$ and $(y_n)$ iterates are bounded almost surely. Proposition~\ref{prop:stability} here proves it for $(x_n),$ while, for $(y_n),$ it follows easily from \cite[Chapter~3, Theorem~7]{borkar2009stochastic} due to its linear nature. Assumptions A6 and A7 hold if the contributions of the additive noise terms are eventually negligible. This can be established as in \cite[Chapter~2, (2.19)]{borkar2009stochastic}, which holds in our case because our iterates are bounded a.s. and the noise growth rate condition of (2.13)  trivially holds in our context. Assumptions A9 and A11 hold, if for each fixed $x,$ the limiting DI for $(y_n)$ has a unique GASE. As discussed in Section~\ref{s:Proof.Outline}, in our case, the dynamics of $(y_n)$ is not influenced by the value of $x$ and $\{\bE[Y(i)]: i \in M^c\}$ is the global attractor for any $x.$ Finally, Assumption A10 requires that $(x_n)$'s limiting DI has a unique global attractor. We established this in Statement~\ref{st:GASE} of our Theorem~\ref{thm:SRI.Main.Result}.

\section{Discussions and Future Directions}
In this work, we developed a fully-asynchronous algorithm for mean estimation in the presence of adversaries. Thereafter, we developed a novel DI-based two-timescale analysis to rigorously show its a.s. convergence. We now discuss some simple extensions of our work, where we can relax certain assumptions.

\textbf{Non-zero kernel:} The condition \eqref{e:FTD.Cond} fails for all matrices $A$ with a non-zero kernel. Thus, Theorem~\ref{thm:SRI.Main.Result} cannot be used for fat matrices or tall matrices with non full rank. However, we can obtain a similar result by relaxing condition \eqref{e:FTD.Cond} to hold only for points outside the kernel of $A$. Note that in this case, there are several $x \in \bR ^d$ such that $Ax=\bE Y$. Under this modified assumption, it can be shown that the DI always converges to one such point. To see this, the function $\frac{1}{2}\|x-\mu\|_2^2$, with $\mu$ as solution of $Ax=\bE Y$, would remain a Lyapunov function in this case. Applying a variant of LaSalle's invariance theorem would then give us that the DI converges to an invariant subset of $\{x:Ax=\bE Y\}$.

\textbf{Perturbed samples:} Suppose that, instead of being provided samples of $Y(i)=a_i^\tr X$, we only have access to samples of form $Y(i)=a_i^\tr X+b(i)$, where $b(i)$ is some random or deterministic perturbation. The only condition imposed on $b(i)$ is that its magnitude remains bounded by some constant $B$ for each $i$. We can extend the result in Theorem \ref{thm:SRI.Main.Result} to this setting using similar arguments as discussed in the previous case. However, the Lyapunov function would need to be re-defined and may have discontinuous derivatives.    

\bibliographystyle{splncs04} 
\bibliography{references.bib}

\end{document}